%% file: samplepaper.tex
\documentclass[runningheads]{llncs}
\usepackage[T1]{fontenc}
\usepackage{graphicx}
\usepackage{xcolor}
\usepackage{amsmath}
\usepackage{amssymb}

\usepackage{enumitem}
\usepackage{tikz}
\usepackage{booktabs}

\begin{document}
\title{Circuit Complexity of Hierarchical \\ Knowledge Tracing and Implications for Log-Precision Transformers}
\titlerunning{Circuit Complexity of Hierarchical Knowledge Tracing}
\author{Naiming Liu\inst{1} \and
Richard Baraniuk\inst{1} \and 
Shashank Sonkar\inst{2}
}
\authorrunning{N. Liu et al.}
%
\institute{Rice University \and
University of Central Florida \\ 
\email{nl35@rice.edu, shashank.sonkar@ucf.edu}
}
\maketitle              

\begin{abstract}
Knowledge tracing models mastery over interconnected concepts, often organized by prerequisites. We analyze hierarchical prerequisite propagation through a circuit-complexity lens to clarify what is provable about transformer-style computation on deep concept hierarchies. Using recent results that log-precision transformers lie in logspace-uniform $\mathsf{TC}^0$, we formalize prerequisite-tree tasks including recursive-majority mastery propagation. Unconditionally, recursive-majority propagation lies in $\mathsf{NC}^1$ via $O(\log n)$-depth bounded-fanin circuits, while separating it from uniform $\mathsf{TC}^0$ would require major progress on open lower bounds. Under a monotonicity restriction, we obtain an unconditional barrier: alternating ALL/ANY prerequisite trees yield a strict depth hierarchy for \emph{monotone} threshold circuits. Empirically, transformer encoders trained on recursive-majority trees converge to permutation-invariant shortcuts; explicit structure alone does not prevent this, but auxiliary supervision on intermediate subtrees elicits structure-dependent computation and achieves near-perfect accuracy at depths 3--4. These findings motivate structure-aware objectives and iterative mechanisms for prerequisite-sensitive knowledge tracing on deep hierarchies.
\end{abstract}

\input{intro}

\input{related}

\input{proof}

\input{results}

\input{short_discussion_conclusion}

%

%
%
%
\bibliographystyle{splncs04}
\bibliography{custom}

\end{document}

%% file: intro.tex
\section{Introduction}

Knowledge tracing, modeling how learners master interconnected concepts over time, is fundamental to student modeling, intelligent tutoring systems, and adaptive learning \cite{corbett1994knowledge,sonkar2020qdkt,liu2022open}. A central challenge is capturing how mastery of prerequisite concepts propagates through a concept hierarchy to enable (or block) mastery of more advanced concepts. While a large body of work has explored neural architectures for knowledge tracing, the \emph{computational} capabilities and limitations of these models remain poorly characterized \cite{piech2015deep,zhang2017dynamic,sonkar2023deduction,ghosh2020context,worden2026foundationalassist}.

This paper takes a circuit-complexity perspective on hierarchical prerequisite reasoning. Recent results show that \emph{log-precision transformers}, transformers whose activations use $O(\log n)$ bits on inputs of length $n$, can be simulated by logspace-uniform constant-depth threshold circuits, i.e., they lie in logspace-uniform $\mathsf{TC}^0$ \cite{merrill2023parallelism,merrill2}. This connection suggests a natural route to transformer limitations: exhibit knowledge-tracing functions that lie beyond (uniform) $\mathsf{TC}^0$. However, proving such lower bounds for \emph{general} $\mathsf{TC}^0$ is notoriously difficult and is intertwined with major open questions in circuit complexity, for example separating $\mathsf{TC}^0$ from $\mathsf{NC}^1$ \cite{vollmer1999introduction}.

We formalize prerequisite propagation on deep concept hierarchies using balanced prerequisite trees and study what can be proved \emph{unconditionally} versus what remains open for standard (non-monotone) transformers. As a baseline, we define a natural hierarchical mastery rule in which each internal concept is mastered if a majority of its prerequisites are mastered; computing the root mastery then corresponds to evaluating a depth-$\Theta(\log n)$ majority formula, placing the task in $\mathsf{NC}^1$ \cite{vollmer1999introduction}. Showing that this function is \emph{not} in (uniform) $\mathsf{TC}^0$ would yield an immediate limitation for log-precision transformers, but would also constitute a significant circuit lower bound beyond current techniques.

To provide unconditional evidence that layered prerequisite structure can resist shallow parallelization, we also analyze an \emph{alternating} prerequisite model (ALL/ANY requirements), which corresponds to evaluation of alternating $\wedge/\vee$ trees. From an educational modeling perspective, prerequisite aggregation is naturally \emph{monotone} in prerequisite mastery: holding the curriculum fixed, mastering additional prerequisites should not reduce a learner's readiness for an advanced concept. This monotonicity is a common desideratum in mastery models because it aligns with the semantics of prerequisite structure and improves interpretability. We therefore analyze an alternating ALL/ANY prerequisite rule under \emph{monotone} threshold computation as a principled restricted setting in which unconditional depth lower bounds are known. Here, classical results establish a strict depth hierarchy for \emph{monotone} threshold circuits: increasing prerequisite depth provably increases the power required, and reducing depth forces superpolynomial (indeed exponential) size \cite{yao1989_circuits_local,hastad_goldmann1991}. While these monotone lower bounds do not directly apply to standard transformers (which are not monotone), they delineate a principled barrier and clarify which additional restrictions would be needed to obtain unconditional transformer impossibility results.

Our contributions provide a landscape view of prerequisite-tree knowledge tracing:

\begin{enumerate}
    \item \textbf{Formalization and complexity accounting.}
    We formalize prerequisite propagation on balanced concept trees under natural mastery rules, including recursive majority, and connect these tasks to standard circuit-evaluation problems.

    \item \textbf{What is provable today.}
    Unconditionally, for fixed arity $k$, recursive-majority prerequisite propagation is computable by logarithmic-depth bounded-fanin circuits, placing it in $\mathsf{NC}^1$; in contrast, proving separation from (uniform) $\mathsf{TC}^0$ would require major progress on circuit lower bounds.

    \item \textbf{Barriers and empirical diagnostics.}
    For an alternating ALL/ANY prerequisite rule, classical results yield a strict depth hierarchy for \emph{monotone} threshold circuits, showing that layered prerequisite structure can resist shallow parallelization under monotonicity constraints. Empirically, transformer encoders trained on recursive-majority trees learn permutation-invariant shortcuts under root-only supervision; explicit structure alone does not prevent this, but auxiliary supervision on intermediate subtrees elicits structure-dependent computation at moderate depths.
\end{enumerate}

Taken together, these results clarify both the promise and the limits of current theory. Hierarchical prerequisite reasoning naturally induces logarithmic-depth computation, and unconditional lower bounds for general $\mathsf{TC}^0$ remain out of reach. Our empirical results complement this view by showing that, even with explicit structure, models can learn shortcuts unless training rewards intermediate prerequisite propagation.

%% file: related.tex
\section{Related Work}

Our work connects three threads: (i) circuit-complexity characterizations of transformer-like architectures, (ii) modeling choices in knowledge tracing (especially prerequisite structure), and (iii) neural mechanisms for recursive / hierarchical computation.

\subsection{Theoretical Analysis of Transformer Computation}

A growing line of work characterizes transformers via circuit complexity. Merrill and Sabharwal~\cite{merrill2023parallelism} show that \emph{log-precision} transformers---those whose activations use $O(\log n)$ bits on length-$n$ inputs---can be simulated by logspace-uniform constant-depth threshold circuits, placing them within logspace-uniform $\mathsf{TC}^0$. Related analyses of restricted or saturated transformer variants connect attention-based computation to constant-depth threshold circuit families~\cite{merrill2}. These results motivate a principled route to expressivity limitations: if a target function can be shown to lie beyond (uniform) $\mathsf{TC}^0$, then it is unreachable by log-precision transformers in this single-pass setting.

Other work studies how specific architectural choices (e.g., attention constraints, normalization) affect expressivity on formal-language tasks~\cite{hahn2020limitations,chiang2022parity}. Strobl et al.~\cite{strobl2024survey} provide a useful synthesis, which harmonizes expressivity results across transformer variants and situates them in the standard circuit hierarchy $\mathsf{AC}^0 \subset \mathsf{TC}^0 \subseteq \mathsf{NC}^1 \subseteq \mathsf{L}$. Importantly, while many natural recursive computations (e.g., Boolean formula evaluation) lie in $\mathsf{NC}^1$, proving that such problems are \emph{not} computable by general $\mathsf{TC}^0$ circuits would imply major open lower bounds (e.g., separating $\mathsf{TC}^0$ from $\mathsf{NC}^1$). Our paper is explicitly motivated by this gap: we identify educationally meaningful hierarchical reasoning tasks that naturally land in $\mathsf{NC}^1$, and we clarify what is currently provable versus what remains open for general $\mathsf{TC}^0$-style transformer computation.

\subsection{Knowledge Tracing Methods and Prerequisite Structure}

Knowledge tracing (KT) began with Bayesian Knowledge Tracing (BKT)~\cite{corbett1994knowledge}, which models binary mastery via hidden states. Deep learning approaches shifted KT toward representation learning from interaction sequences. Deep Knowledge Tracing (DKT)~\cite{piech2015deep} applies RNNs to student-event sequences, while DKVMN~\cite{zhang2017dynamic} introduces key--value memory to separate concept representations from student mastery states.

Transformers have since become prominent in KT. SAKT~\cite{pandey2019sakt} uses self-attention to retrieve relevant past interactions; AKT~\cite{ghosh2020context} incorporates monotonicity and temporal effects. Subsequent work explores encoder--decoder variants~\cite{choi2020saint} and relation-aware attention~\cite{pandey2020rkt} illustrating that architectural choices materially affect performance.

A parallel line explicitly models concept structure and prerequisites. Graph-based KT methods such as GKT~\cite{nakagawa2019gkt} and GIKT~\cite{yang2020gikt} incorporate concept graphs via message passing or graph convolutions. Other approaches directly encode prerequisite influence or constraints, including SKT~\cite{tong2020skt} and prerequisite-driven extensions to sequence models~\cite{chen2018prerequisite}.

Despite substantial empirical progress, most KT work evaluates architectures experimentally rather than characterizing which \emph{structural} reasoning patterns (e.g., deep hierarchical prerequisite propagation) are compatible with their underlying computational model. Our work addresses this by formalizing prerequisite-tree propagation as a circuit-evaluation task and relating it to the transformer--$\mathsf{TC}^0$ literature.

\subsection{Recursive and Hierarchical Computation in Neural Networks}

The ability of neural networks to implement recursive computation has long been studied. Siegelmann and Sontag~\cite{siegelmann1995rnn} show that RNNs with unbounded precision can simulate general computation, though such results rely on assumptions that do not directly match fixed-precision modern implementations.

For structured domains, models with explicit inductive bias often match the computational structure of the task. Tree-LSTMs~\cite{tai2015treelstm} compute representations bottom-up along a given tree, aligning naturally with prerequisite-tree propagation. For general graphs, message-passing GNNs have inherent depth requirements for long-range dependencies: global aggregation typically demands depth proportional to the graph diameter, and practical depth can be limited by representation mixing effects~\cite{loukas2020gnn}. These observations parallel the KT setting, where prerequisite influence may need to traverse long concept chains.

Alternative architectures---state space models~\cite{gu2023mamba,merrill2024mamba} and memory-augmented networks~\cite{graves2014ntm,graves2016dnc}---offer different computational tradeoffs, though their expressivity under bounded precision is an active area of study.

Our contribution in this context is to connect an educationally motivated recursive reasoning task (hierarchical prerequisite propagation) to the circuit-complexity landscape: we give clean $\mathsf{NC}^1$ upper bounds for natural prerequisite-tree rules, identify unconditional barriers in restricted (monotone) threshold settings, and clarify which additional lower bounds would be required to turn these observations into unconditional limitations for general log-precision transformers.

%% file: proof.tex

\section{Background: Circuit Complexity and Transformers}

Circuit complexity studies computation via families of Boolean circuits and characterizes
problems by resources such as circuit size and depth \cite{arora2009computational,vollmer1999introduction}.
Circuit \emph{depth} is the length of the longest path from any input bit to the output.

The class $\mathsf{TC}^0$ consists of Boolean functions computable by \emph{uniform} families of
constant-depth, polynomial-size circuits with unbounded fan-in threshold (majority) gates
\cite{hajnal1993threshold}. Logspace-uniformity ensures the circuit for each input length can be
constructed by a logspace algorithm \cite{barrington1989uniform}.

Recent work shows that log-precision transformers, meaning transformers whose activations
use $O(\log n)$ bits on inputs of length $n$, \emph{can be simulated by} logspace-uniform
constant-depth threshold circuits. That is, they lie in logspace-uniform $\mathsf{TC}^0$
\cite{merrill2023parallelism,merrill2}.
Thus, any separation from logspace-uniform $\mathsf{TC}^0$ would immediately yield
a limitation for log-precision transformers.

\section{Prerequisite-Tree Knowledge Tracing: Upper Bound, Monotone Barrier, and an Open Question}

We formalize prerequisite propagation on deep concept hierarchies and clarify what is known
unconditionally versus what remains open for \emph{general} $\mathsf{TC}^0$ (and hence standard
log-precision transformers).

\subsection{Definitions}

\paragraph{Concept tree and inputs.}
We consider prerequisite \emph{trees}. Let $n$ denote the number of leaf concepts (input mastery bits),
and let $N$ denote the total number of concepts (nodes).

Fix an integer $k \ge 3$. Consider a perfectly balanced $k$-ary tree of depth $d$ (root at depth $0$),
where each internal node has exactly $k$ children. The number of leaves is $n = k^d$ (equivalently, $d = \log_k n$).
The total number of nodes is
\[
N = 1 + k + k^2 + \cdots + k^d = \frac{k^{d+1}-1}{k-1} = \frac{kn-1}{k-1} = O(n).
\]

\paragraph{Majority prerequisite rule.}
Leaves are labeled by input mastery bits in $\{0,1\}$.
Each internal node computes the (strict) majority of its $k$ children:
\[
\mathrm{MAJ}_k(x_1,\ldots,x_k)=1 \quad\text{iff}\quad \sum_{i=1}^k x_i \ge \left\lfloor \frac{k}{2}\right\rfloor + 1.
\]
(For odd $k$ this is the usual majority; for even $k$ this fixes a tie-breaking convention.)
Let $\mathrm{KT}_{\mathrm{MAJ}}$ denote the problem of computing the root value.

\subsection{Unconditional upper bound: $\mathrm{KT}_{\mathrm{MAJ}} \in \mathsf{NC}^1$}

\begin{theorem}[Upper Bound]
\label{thm:ktmaj_in_nc1}
For every fixed $k \ge 3$, $\mathrm{KT}_{\mathrm{MAJ}}$ on balanced $k$-ary trees with $n$ leaves
is computable in $\mathsf{NC}^1$.
\end{theorem}

\begin{proof}
The root value is the evaluation of a depth-$d=\Theta(\log n)$ formula whose internal gates are
constant-fanin majorities $\mathrm{MAJ}_k$. Because $k$ is a fixed constant, each $\mathrm{MAJ}_k$
can be implemented by a constant-size bounded-fanin Boolean subcircuit; substituting these yields a
bounded-fanin Boolean circuit of depth $O(d)=O(\log n)$ and polynomial size. Hence the problem lies in $\mathsf{NC}^1$
\cite{vollmer1999introduction}.
\end{proof}

\paragraph{Important caveat.}
Theorem~\ref{thm:ktmaj_in_nc1} does \emph{not} imply $\mathrm{KT}_{\mathrm{MAJ}} \notin \mathsf{TC}^0$.
Proving $\mathrm{KT}_{\mathrm{MAJ}} \notin$ (general) $\mathsf{TC}^0$ would be a major circuit lower bound and would
separate $\mathsf{TC}^0$ from $\mathsf{NC}^1$ \cite{vollmer1999introduction}.

\subsection{A monotone barrier (unconditional): alternating prerequisite trees resist shallow \emph{monotone} threshold circuits}

We next record an unconditional limitation for a \emph{restricted} model, namely
\emph{monotone} threshold circuits (no negated inputs and nonnegative weights in threshold gates).

\paragraph{Why monotonicity is a meaningful restriction.}
When mastery is represented as binary prerequisite indicators, a natural desideratum is \emph{monotonicity}:
mastering additional prerequisites should not decrease predicted readiness for a target concept.
This aligns with the semantics of prerequisite structure and supports interpretability for prerequisite-aware interventions.
Standard transformers are not monotone in general due to negative weights and non-monotone feature interactions.

We use this restriction as a normative baseline for prerequisite aggregation, rather than as a direct model of transformer computation.
Even in this restricted setting, layered prerequisite structure can resist shallow parallelization.

\paragraph{Alternating ALL/ANY prerequisite trees.}
Define $\mathrm{KT}_{\wedge/\vee}$ as prerequisite propagation on a tree whose internal nodes alternate between:
(i) ALL prerequisites required (AND), and (ii) ANY prerequisite sufficient (OR).
This is exactly evaluation of an alternating $\wedge/\vee$ tree on the leaf bits.
This task is monotone in the leaf mastery bits, so it aligns with the monotonicity desideratum above.

\begin{theorem}[Monotone threshold depth hierarchy via alternating prerequisite formulas]
\label{thm:monotone_threshold_depth_hierarchy}
For every integer $t \ge 2$, there exists an explicit monotone Boolean function $f_t$ that is
computed by a depth-$t$ alternating $\wedge/\vee$ \emph{read-once formula} of linear size (allowing
unbounded fan-in gates), but for which every depth-$(t-1)$ \emph{monotone} threshold circuit requires
size $\exp\!\bigl(n^{\Omega(1/t)}\bigr)$.
\end{theorem}

\begin{proof}[Justification and references]
Yao introduced monotone threshold circuits and proved a strict depth hierarchy by exhibiting
explicit monotone functions computable at larger depth that require exponential size when the depth is reduced
\cite{yao1989_circuits_local}.
H{\aa}stad and Goldmann strengthened and streamlined these results using an explicit family $f_t$ defined by
a depth-$t$ alternating $\wedge/\vee$ read-once formula of linear size (with unbounded fan-in), and proving
lower bounds of the form $\exp(n^{\Omega(1/t)})$ for depth-$(t-1)$ monotone threshold circuits
\cite{hastad_goldmann1991}.
\end{proof}

\paragraph{How to read Theorem~\ref{thm:monotone_threshold_depth_hierarchy}.}
This theorem states that \emph{no fixed constant depth} of monotone threshold circuits can, in general,
compress away the layered prerequisite structure captured by alternating ALL/ANY trees.
(It does \emph{not} claim that the specific balanced $k$-ary tree from our $\mathrm{KT}_{\mathrm{MAJ}}$ definition is hard in this model.)

\paragraph{Reduction from $\wedge/\vee$ trees to majority trees (restriction).}
We now connect this monotone barrier to majority-tree prerequisite propagation.

\begin{lemma}[AND/OR as restricted ternary majority]
\label{lem:andor_as_maj3}
For Boolean inputs $a,b\in\{0,1\}$,
\[
\mathrm{MAJ}_3(a,b,0)= a \wedge b,
\qquad
\mathrm{MAJ}_3(a,b,1)= a \vee b.
\]
\end{lemma}

\begin{corollary}[Monotone lower bound transfers to majority-tree KT]
\label{cor:maj_hard_monotone}
Fix any $t \ge 2$ and let $f_t$ be as in Theorem~\ref{thm:monotone_threshold_depth_hierarchy}.
Replacing each $\wedge/\vee$ gate in the defining alternating tree for $f_t$ using Lemma~\ref{lem:andor_as_maj3}
(and adding constant leaves) yields a ternary-majority tree function $g_t$ such that any depth-$(t-1)$
polynomial-size monotone threshold circuit for $g_t$ would imply one for $f_t$.
In particular, $g_t$ also requires superpolynomial size at depth $(t-1)$ in the monotone threshold model.
\end{corollary}

\begin{proof}
Replacing each $\wedge/\vee$ internal node by $\mathrm{MAJ}_3(\cdot,\cdot,0/1)$ yields a ternary-majority tree whose
function restricts (by fixing the added constant leaves) to the original $\wedge/\vee$ tree function.
Monotone threshold circuits are closed under restrictions, so any depth-$(t-1)$ monotone threshold circuit for $g_t$
would give one for $f_t$, contradicting Theorem~\ref{thm:monotone_threshold_depth_hierarchy}.
\end{proof}

\paragraph{On the monotone vs.\ general $\mathsf{TC}^0$ gap.}
Theorems~\ref{thm:monotone_threshold_depth_hierarchy}--\ref{cor:maj_hard_monotone} are unconditional but apply only to
\emph{monotone} threshold circuits. In contrast, these (and related) tree-evaluation families can be far easier for
\emph{general} threshold circuits, and obtaining comparable lower bounds for general $\mathsf{TC}^0$ remains open.
For example, H{\aa}stad--Goldmann note that their monotone-hard functions admit much shallower \emph{general} threshold circuits
via known simulations \cite{hastad_goldmann1991,allender1989_threshold}.

\subsection{Implications for log-precision transformers (conditional)}

\begin{corollary}[Conditional transformer implication]
\label{cor:transformer_conditional}
If $\mathrm{KT}_{\mathrm{MAJ}}$ (recursive majority on balanced trees) is not computable by logspace-uniform
(general) $\mathsf{TC}^0$ circuits, then log-precision transformers cannot compute $\mathrm{KT}_{\mathrm{MAJ}}$.
\end{corollary}

\begin{proof}
Log-precision transformers lie in logspace-uniform $\mathsf{TC}^0$ \cite{merrill2023parallelism,merrill2}.
The claim follows immediately.
\end{proof}

\paragraph{Takeaway.}
Unconditionally, $\mathrm{KT}_{\mathrm{MAJ}} \in \mathsf{NC}^1$, and alternating prerequisite-tree propagation exhibits strong
limitations for \emph{monotone} threshold circuits in the form of a strict depth hierarchy.
Any unconditional limitation for standard (log-precision) transformers would require progress on open lower bounds for
\emph{general} $\mathsf{TC}^0$.

%% file: results.tex
\section{Empirical Analysis}

We complement the theoretical analysis with controlled experiments probing whether standard transformer encoders \emph{learn} hierarchical prerequisite composition on recursive majority trees.

\subsection{Experimental Setup}

\paragraph{Task.}
We instantiate the $\mathrm{KT}_{\mathrm{MAJ}}$ task from Section~3. Given $n = 3^d$ binary leaf mastery bits, the label is the root value of a balanced ternary tree in which each internal node outputs $\mathrm{MAJ}_3$ of its three children. Labels are computed deterministically by bottom-up evaluation.

\paragraph{Design rationale: in-distribution evaluation.}
To avoid conflating hierarchical computation with length extrapolation, we evaluate models \emph{in distribution} at a fixed depth. Concretely, for each depth $d \in \{3,4,5,6\}$ (corresponding to $n \in \{27,81,243,729\}$), we train a fresh model and evaluate on held-out examples drawn from the same distribution. Thus, any performance gap reflects the solution found by training under fixed model class and optimization, rather than out-of-distribution generalization.

\paragraph{Models.}
We compare:
\begin{enumerate}
    \item \textbf{Transformer encoder.} A standard encoder with 4 layers, hidden size 128, 4 attention heads, feedforward size 512, and sinusoidal positional encodings. Inputs are the leaf bits tokenized as \texttt{0}/\texttt{1} with a prepended \texttt{[CLS]} token; the final \texttt{[CLS]} representation is fed to a linear classifier (approximately 794K parameters).

    \item \textbf{MLP baseline (global sum only).} A 3-layer MLP whose sole input is the normalized leaf sum $(\sum_i x_i)/n$. This baseline cannot represent any position- or structure-dependent function; it captures what can be achieved using only permutation-invariant aggregate statistics.
\end{enumerate}
We also report an \textbf{oracle} baseline that computes the exact recursive majority, achieving 100\% accuracy by construction.

\paragraph{Training.}
For each depth, we generate 20{,}000 training examples, 5{,}000 validation examples, and 5{,}000 test examples, with leaves sampled i.i.d.\ from $\mathrm{Bernoulli}(0.5)$. Models are trained with AdamW (learning rate $3\times10^{-4}$, weight decay 0.01) and early stopping on validation accuracy (patience 15 epochs). All experiments use a single NVIDIA GPU.

\paragraph{Permutation diagnostic.}
To probe whether a trained model uses positional information, we evaluate it on \emph{permuted} test inputs: for each test example we uniformly shuffle leaf positions while keeping the original label (computed on the unpermuted tree). A model whose accuracy is unchanged under this perturbation behaves as a permutation-invariant predictor on this task instance. We also report a \textbf{permuted-oracle} score: the accuracy of the true tree evaluator applied to permuted leaves, compared to the original label, which quantifies how much shuffling disrupts the task-relevant structure.

\subsection{Results}

\paragraph{Transformer performance tracks the sum-only baseline.}
Table~\ref{tab:main_results} shows test accuracy across depths. The transformer matches the sum-only MLP to within noise at every depth, while the oracle achieves 100\% accuracy. Both learned models achieve 75--80\% accuracy at shallow depths and degrade smoothly to $\approx 68\%$ at depth 6.

\begin{table}[t]
\centering
\caption{Test accuracy (\%) on recursive majority trees. The transformer closely matches a baseline that only observes the global leaf sum, while the oracle achieves 100\% at all depths.}
\label{tab:main_results}
\begin{tabular}{ccccc}
\toprule
\textbf{Depth} & \textbf{Leaves} & \textbf{Transformer} & \textbf{MLP (sum)} & \textbf{Oracle} \\
\midrule
3 & 27  & 79.7 & 80.0 & 100.0 \\
4 & 81  & 75.8 & 75.7 & 100.0 \\
5 & 243 & 70.7 & 70.6 & 100.0 \\
6 & 729 & 68.4 & 68.3 & 100.0 \\
\bottomrule
\end{tabular}
\end{table}

The close agreement with the sum-only baseline suggests that, under this training setup, the transformer learns a predictor dominated by permutation-invariant aggregate information (e.g., the leaf sum) rather than an exact tree-structured computation. Intuitively, the leaf sum is correlated with the root label under $\mathrm{Bernoulli}(0.5)$ leaves, but this correlation weakens with depth, consistent with the observed degradation.

\paragraph{Permutation invariance supports a permutation-invariant solution.}
Table~\ref{tab:permutation} reports the permutation diagnostic. Transformer accuracy is essentially unchanged when leaf positions are randomly permuted (changes within $\pm 0.3\%$), mirroring the behavior of the sum-only MLP. In contrast, the permuted-oracle score drops with depth (from 65.8\% at depth 4 to 56.4\% at depth 6), indicating that shuffling materially disrupts the task-relevant hierarchical structure.

\begin{table}[t]
\centering
\caption{Permutation diagnostic. Transformer accuracy is unchanged under random leaf permutation, consistent with reliance on permutation-invariant aggregate features. The permuted-oracle score shows that permutation disrupts underlying tree structure.}
\label{tab:permutation}
\resizebox{0.85\linewidth}{!}{
\begin{tabular}{ccccc}
\toprule
\textbf{Depth} & \textbf{Trans.\ (orig)} & \textbf{Trans.\ (perm)} & \textbf{MLP (sum)} & \textbf{Perm.\ Oracle} \\
\midrule
4 & 76.0 & 76.0 & 76.0 & 65.8 \\
5 & 70.0 & 70.2 & 70.5 & 63.0 \\
6 & 68.4 & 68.2 & 67.7 & 56.4 \\
\bottomrule
\end{tabular}
}
\end{table}

Together, Tables~\ref{tab:main_results}--\ref{tab:permutation} indicate that, in this controlled setting, the trained transformer behaves similarly to a permutation-invariant predictor: it does not appear to exploit positional encodings to recover the fixed ternary grouping required for exact recursive evaluation.

\paragraph{Interpretation and limitations.}
These experiments do \emph{not} establish a representational impossibility for transformers. In particular, our theory section highlights that $\mathrm{KT}_{\mathrm{MAJ}} \in \mathsf{NC}^1$ and that separations from (uniform) $\mathsf{TC}^0$ remain open. Rather, the experiments provide evidence about \emph{learned solutions}: with standard training and a fixed-capacity encoder, gradient-based optimization converges to a shallow, permutation-invariant shortcut that attains moderate accuracy without implementing the hierarchical computation needed for perfect performance.

This behavior is consistent with the broader perspective of the paper: hierarchical prerequisite propagation induces depth-$\Theta(\log n)$ composition, and absent architectural bias toward recursive structure, models may default to correlations that are easier to capture from data but insufficient for exact prerequisite-tree evaluation.

\subsection{Can Structural Scaffolding Help?}

\paragraph{Scaffolding as an intervention on learned solutions.}
The preceding experiments show that, under standard root-only supervision, transformer encoders converge to a permutation-invariant shortcut. This observation alone does not distinguish representational limitations from learning dynamics. We therefore test whether (i) making the hierarchy explicit in the input and (ii) adding intermediate supervision can steer the same architecture toward a structure-dependent solution.

\paragraph{Level-tagged structural encoding.}
We modify the input representation to expose the ternary grouping structure. Instead of a flat sequence $x_1 x_2 \ldots x_n$, we insert level-tagged separator tokens $]_k$ that mark the end of each \emph{height-$k$} subtree (i.e., $k$ aggregation steps above the leaves) in a fixed left-to-right traversal. For a depth-$d$ ternary tree, the resulting sequence interleaves leaf bits and separators, ending with a final $]_d$ token corresponding to the root subtree. This encoding makes subtree boundaries explicit in token identity, rather than requiring the model to infer them from positional encodings alone.

\paragraph{Auxiliary supervision on intermediate subtrees.}
Structure alone may be insufficient if the training objective provides no incentive to represent intermediate values. We therefore add an auxiliary loss at separator positions: at each $]_k$, the model predicts the majority value of the subtree that just closed. This encourages learning intermediate computations aligned with the hierarchical structure. We train with $\mathcal{L} = \mathcal{L}_{\mathrm{root}} + \lambda \mathcal{L}_{\mathrm{aux}}$ with $\lambda=1.0$.\footnote{Aux Acc is computed over all separator positions (including the root-level separator $]_d$) using a shared auxiliary classifier head; root accuracy is computed from a separate classifier head applied to the \texttt{[CLS]} token for consistency with the flat transformer baseline.}

\paragraph{Scaffold results.}
Table~\ref{tab:scaffold_main} summarizes the scaffold experiments. Two findings stand out. First, adding structure without auxiliary supervision (Struct Only) does not improve over the flat baseline, indicating that explicitly marking subtree boundaries is not enough to prevent shortcut learning under a root-only objective. Second, combining structure with auxiliary supervision (Struct+Aux) yields near-perfect accuracy at depths 3--4 (99.96\% and 99.36\%), along with high separator-prediction accuracy (Aux Acc $\ge 99.8\%$). At depth~5, Struct+Aux improves root accuracy from 70.8\% to 77.2\% and achieves high average separator accuracy (98.7\%). We emphasize that Aux Acc aggregates over \emph{all} separators (many low-level subtrees and a single root-level separator), so it is dominated by lower-level predictions and should not be interpreted as the model predicting the root separator with 98.7\% accuracy.

\begin{table}[t]
\centering
\caption{Scaffold experiment results. Structure alone (Struct Only) does not improve over the flat baseline. Adding auxiliary supervision (Struct+Aux) yields large gains at shallow depths and substantially increases separator-prediction accuracy.}
\label{tab:scaffold_main}
\begin{tabular}{cccccc}
\toprule
\textbf{Depth} & \textbf{Leaves} & \textbf{Flat} & \textbf{Struct Only} & \textbf{Struct+Aux} & \textbf{Aux Acc} \\
\midrule
3 & 27  & 80.0 & 80.0 & 99.96 & 99.9 \\
4 & 81  & 75.7 & 75.7 & 99.36 & 99.8 \\
5 & 243 & 70.8 & 70.8 & 77.2  & 98.7 \\
6 & 729 & 68.3 & 68.5 & 68.7  & 53.0 \\
\bottomrule
\end{tabular}
\end{table}

\paragraph{Permutation sensitivity indicates structure-dependent behavior.}
Table~\ref{tab:scaffold_perm} reports the permutation diagnostic for the scaffolded models. For Struct+Aux at depths 3--5, permuting leaf values while keeping separators fixed causes substantial accuracy drops (up to 32.7 points at depth~4), in contrast to the flat baseline whose accuracy is unchanged under permutation. This indicates that, with auxiliary supervision, the learned solution depends on the alignment between leaf positions and subtree boundaries, rather than a purely permutation-invariant aggregate statistic.

\begin{table}[t]
\centering
\caption{Permutation diagnostic for scaffold models. Accuracy drop under leaf permutation indicates structure-dependent computation.}
\label{tab:scaffold_perm}
\resizebox{0.9\linewidth}{!}{
\begin{tabular}{ccccc}
\toprule
\textbf{Depth} & \textbf{Flat (perm drop)} & \textbf{Struct+Aux (orig)} & \textbf{Struct+Aux (perm)} & \textbf{Drop} \\
\midrule
3 & $+$0.0 & 99.96 & 72.4 & $+$27.6 \\
4 & $+$0.0 & 99.36 & 66.7 & $+$32.7 \\
5 & $+$0.0 & 77.2  & 65.6 & $+$11.6 \\
6 & $-$0.1 & 68.7  & 68.2 & $+$0.5 \\
\bottomrule
\end{tabular}
}
\end{table}

\paragraph{An empirical limit at depth 6 under this configuration.}
At depth~6 (729 leaves; 364 separator positions), Struct+Aux does not improve root accuracy over the baseline, and separator-prediction accuracy drops to 53\%---barely above chance---with no permutation sensitivity. Under our training budget and architecture (4-layer, $d_{\mathrm{model}}{=}128$), we therefore do not observe successful learning of the hierarchical computation at this depth. Determining whether this reflects optimization difficulty, insufficient capacity, or the need for different structural inductive bias is an important direction for future work.

%% file: short_discussion_conclusion.tex
\section{Discussion and Conclusion}

Prerequisite-based knowledge tracing fundamentally requires an \emph{aggregation} operation: given mastery signals over prerequisite concepts, infer readiness for a target concept. Real curricula rarely match the extremes of requiring \emph{all} prerequisites (too strict) or \emph{any} prerequisite (too lenient); instead they behave like a threshold rule—students need \emph{enough} prerequisite mastery to progress. Modeling this as recursive threshold aggregation over a concept hierarchy yields a principled abstraction of prerequisite propagation. Our complexity analysis clarifies that this propagation is inherently layered: for fixed arity, evaluating recursive-majority prerequisite trees is a logarithmic-depth computation (in $\mathsf{NC}^1$), and collapsing deep prerequisite structure into a constant-depth computation would require resolving major open circuit lower bounds. Thus, rather than claiming an impossibility for transformers, we use complexity as a lens to explain why deep prerequisite propagation is a nontrivial capability that may not be recovered from end-task supervision alone.

Empirically, we find that standard transformer encoders trained with \emph{root-only} supervision do not learn prerequisite propagation even on a fully observed hierarchy. Across depths, they converge to permutation-invariant shortcuts that track a sum-only baseline and remain stable under structure-destroying perturbations, leaving a large gap to the exact prerequisite algorithm. This matters for KT: a model can appear accurate while effectively behaving as an \emph{aggregate mastery} predictor that ignores which prerequisites are mastered, undermining interpretability and prerequisite-aware interventions. Crucially, shortcut learning is not inevitable. When we make hierarchy boundaries explicit and add \emph{intermediate concept-level supervision} aligned with subtree mastery, the same architecture switches to structure-dependent behavior and achieves near-perfect propagation at moderate depths, while revealing a regime (deeper hierarchies) where learning breaks under fixed capacity and training budget.

For AIED, the practical takeaway is that prerequisite structure should be treated as a first-class object in KT evaluation and training. Accuracy on held-out responses can be consistent with a model that ignores prerequisite dependencies; therefore, KT benchmarks should include \emph{counterfactual/structure-sensitivity diagnostics} alongside standard metrics. On the modeling side, our results suggest two actionable design principles: (i) incorporate mechanisms that explicitly represent and update intermediate prerequisite mastery (multi-task supervision, latent concept heads, or structured propagation modules), and (ii) support depth-adaptive computation for curricula with long prerequisite chains (e.g., iterative inference, recurrent/looped computation, or hybrid KT models that delegate propagation to a structured component). More broadly, our prerequisite-propagation probe provides a controlled way to stress-test whether KT models truly use hierarchical prerequisite knowledge, and to measure when and how supervision or inductive bias can elicit the intended reasoning.